\documentclass[12pt]{article}

\usepackage{setspace}

\usepackage{geometry}                
\usepackage{graphicx}
\usepackage{amssymb}
\usepackage{amsmath}
\allowdisplaybreaks
\usepackage{mathrsfs} 
\usepackage{color}
\usepackage{hyperref}
\hypersetup{
	colorlinks=true, 
	linktoc=all,     
	linkcolor=blue,  
	citecolor=blue,
}

\usepackage{appendix}
\usepackage{tikz}
\usepackage[numbers]{natbib}

\usetikzlibrary{positioning}
\usetikzlibrary{arrows}

\newtheorem{notation}{Notation}[section] 
\newtheorem{question}{Question}[section] 
\newtheorem{proposition}{Proposition}[section] 

\newtheorem{properties}{Properties}[section]

\newtheorem{theorem}{Theorem}[section]
\newtheorem{corollary}{Corollary}[section]
\newtheorem{lemma}{Lemma}[section]

\newenvironment{proof}{{\noindent\it Proof}\quad}{\hfill $\square$\par} 

\numberwithin{figure}{section}
\numberwithin{equation}{section}

\usepackage{algorithm}
\usepackage{algorithmic}
\usepackage{bm}

\usepackage{amssymb}

\begin{document}

\title{On the Optimal Expressive Power of ReLU DNNs and Its Application in Approximation with Kolmogorov Superposition Theorem}

\author{Juncai He\footnotemark[1]}
\date{} 

\maketitle
\renewcommand{\thefootnote}{\fnsymbol{footnote}} 
\footnotetext[1]{Computer, Electrical and Mathematical Science and Engineering Division, King Abdullah University of Science and Technology, Thuwal 23955, Saudi Arabia (juncai.he@kaust.edu.sa).}

\maketitle

\begin{abstract}
This paper is devoted to studying the optimal expressive power of ReLU deep neural networks (DNNs) and its application in approximation via the Kolmogorov Superposition Theorem. We first constructively prove that any continuous piecewise linear functions on $[0,1]$, comprising $\mathcal O(N^2L)$ segments, can be represented by ReLU DNNs with $L$ hidden layers and $N$ neurons per layer. Subsequently, we demonstrate that this construction is optimal regarding the parameter count of the DNNs, achieved through investigating the shattering capacity of ReLU DNNs. Moreover, by invoking the Kolmogorov Superposition Theorem, we achieve an enhanced approximation rate for ReLU DNNs of arbitrary width and depth when dealing with continuous functions in high-dimensional spaces.
\end{abstract}

\section{Introduction and motivation}
\label{sec:intro}
In the field of deep learning, the efficacy and versatility of neural networks are influenced by two fundamental aspects: expressivity and approximation properties. The expressivity of deep neural networks (DNNs) pertains to their capability to precisely represent a wide array of complex functions. Specifically, the expressive power of ReLU DNNs lies in their ability to capture complicated structures through piecewise linear functions, as these networks are fundamentally constructed using the Rectified Linear Unit (ReLU)~\cite{nair2010rectified} activation function.
On the other hand, approximation properties deal with a network's capacity to accurately approximate target functions, ensuring it can effectively handle complicated functions with precision. Understanding and enhancing both the expressivity and approximation properties of deep learning architectures are crucial endeavors that hold the key to unlocking the full potential of artificial intelligence and advancing the boundaries of deep learning technology.

As for the expressive power of ReLU neural networks (NNs), a highlight property of ReLU DNNs is that they are essentially continuous piecewise linear (CPwL) functions, given that compositions of CPwL functions are still CPwL. Consequently, an important question arises: Are ReLU DNNs capable of accurately reconstructing any piecewise linear function in even a one-dimensional domain, such as $[0,1]$?
Indeed, it has been shown that any piecewise linear function with $N$ segments on $[0,1]$ can be expressed as a ReLU neural network function with just one hidden layer and $N$ parameters~\cite{he2020relu,shen2019nonlinear}. However, the expressive power of ReLU DNNs with two hidden layers holds greater significance than that of ReLU NNs with only one hidden layer. This is due to the crucial consideration of the parameter increase when combining two ReLU NNs.
When combining two ReLU NNs, each with one hidden layer and $N$ neurons, to form a ReLU DNN with two hidden layers and $N$ neurons in each layer, the parameter count increases quadratically, scaling as $\mathcal O(N^2)$. Conversely, when two ReLU DNNs with two hidden layers and $N$ neurons each are combined, the parameter count increases linearly, scaling as $\mathcal O(N^2)$, which is only proportional to the total number of parameters of the ReLU DNN with two hidden layers.
This property of DNNs emphasizes the crucial role of the expressive power in ReLU DNNs with multiple hidden layers, especially when creating deeper structures.

To the best of our knowledge, the following lemma cited from~\cite{shen2019nonlinear} is the pioneering result partially addressing the expressive power of ReLU DNNs with two hidden layers. Essentially, it asserts that ReLU DNNs equipped with two hidden layers and $\mathcal O(N)$ neurons per layer can accurately interpolate $N^2$ grid points on the interval $[0,1]$. However, there exist $N$ subintervals where the corresponding ReLU DNN function fails to maintain linearity, implying that the function is not an exact piecewise linear interpolation of these $N^2$ grid points.
\begin{lemma}[Lemma 2.2 in \cite{shen2019nonlinear}]\label{lemm:interpolation-shen}
For any grid points $\{ 0 = x_0 < x_1 < \cdots < x_{N^2-1} < x_{N^2} = 1 \}$ on $[0,1]$ and sample points $(x_i,y_i) \in \mathbb R^2$ for $i=0:N^2$, there exists a ReLU DNN function $f(x)$ with two hidden layers and $2N+1$ neurons at each layer such that $f(x_i) = y_i$ for all $i=0:N^2$ and $f(x)$ is linear on $[x_{i-1},x_{i}]$ except for $i=kN$ with $k=1:N$.
\end{lemma}
Due to the presence of these intervals where $f(x)$ does not maintain strict linearity, the authors in \cite{lu2021deep} denote them as trifling regions. They subsequently construct an auxiliary ReLU neural network to regulate the variation of $f(x)$ within these trifling regions, a crucial step for $L^\infty$ error estimates.

Consequently, we are interested in the following question, which inquires about the potential for an exact representation for piecewise linear interpolation:
\begin{question}\label{que:N2}
On the interval $[0,1]$, is it possible to reconstruct any piecewise linear functions with $N^2$ segments using ReLU DNNs with two hidden layers and $\mathcal O(N)$ neurons in each layer?
\end{question}

This naturally leads us to pose a corresponding question for ReLU DNNs with general depth and width:
\begin{question}\label{que:NL}
On the interval $[0,1]$, is it possible to reconstruct any piecewise linear functions with $P = \mathcal O(N^2L)$ segments using ReLU DNNs with $L$ hidden layers and $N$ neurons in each layer? Furthermore, is this representation optimal in relation to the number of parameters?
\end{question}
In this paper, we initially present a constructive proof for Question~\ref{que:N2}, employing a symmetric construction approach for ReLU DNNs with two hidden layers. We then extend this expressivity result to ReLU DNNs of arbitrary depth. In addition, we demonstrate that the expressivity of ReLU DNNs for CPwL functions in our result is optimal, particularly with respect to the number of parameters of the ReLU DNNs.

On the other hand, it is very natural to combine the expressivity of ReLU DNNs on one-dimensional space and the Kolmogorov Superposition Theorem (KST). Although there are many different versions of KST, the key concept of KST is to decompose a multivariate function into a composition of several univariate functions. For example, one version of KST in \cite{lorentz1962metric,lorentz1966approximation} shows that any continuous $f(\bm x)$ on $[0,1]^d$ for any $d\ge 2$ can be written as $f(\bm x) = \sum_{k=0}^{2d} g\left(\sum_{i=1}^d\lambda_i\phi_k(x_i)\right)$ for any $\bm x = (x_1, \cdots, x_d)\in[0,1]^d$ where $\phi_k$ and $g$, which are called K-inner or K-base functions and K-outer function, are all continuous univariate functions. By using the expressive power of ReLU DNNs for CPwL functions on $[0,1]$ and their approximation properties, we obtain a generic approximation property of ReLU DNNs with arbitrary depth and width for any continuous functions on $[0,1]^d$. Upon a deeper analysis of KST, we demonstrate that ReLU DNNs, equipped with $L$ hidden layers and $N$ neurons in each layer, can approximate functions within a specific class at a rate of $\mathcal O(N^{-2}L^{-1})$. This finding not only breaks the curse of dimensionality but also achieves the optimal error rate to the best of our knowledge.

The structure of this paper is as follows. In the following Section~\ref{sec:intro}, we provide a discussion on the related work pertaining to our primary findings. In Section~\ref{sec:notation}, we introduce the notations which will be utilized throughout the rest of the paper. In Section~\ref{sec:expres}, we present the constructive proof demonstrating the expressivity of ReLU DNNs. In Section~\ref{sec:optimalrep}, we offer a proof asserting the optimal representation of CPwL functions by ReLU DNNs. In Section~\ref{sec:approx}, we introduce two new approximation results, derived from the Kolmogorov Superposition Theorem. Finally, In Section~\ref{sec:conclusion}, we finish the paper with some concluding remarks.

\subsection{Related work}
\paragraph{Expressive power of NNs}
Regarding the expressive power of ReLU DNNs for CPwL functions, they are essentially equivalent. Specifically, any CPwL function on $\mathbb R^d$ can be represented by a ReLU DNN with no more than $\lceil \log_2(d+1) \rceil$ hidden layers~\cite{arora2018understanding}. However, as indicated in \cite{he2020relu}, this representation could require an extraordinarily large number of parameters. Leveraging the properties of linear finite elements, a more concise and efficient representation was established for any linear finite element function (CPwL function with simplex meshes) in \cite{he2020relu}. Furthermore, for any CPwL function on a uniform mesh on $[0,1]^d$ with a total number of $P$ linear regions (or elements in finite element method terminology~\cite{ciarlet2002finite}), \cite{yarotsky2018optimal} proposed a representation using a very deep ReLU NN architecture. This representation maintained a constant width but needed $\mathcal O(P)$ hidden layers.
However, the question remains open as to whether any CPwL functions on $[0,1]^d$ with $P$ linear regions can be represented by ReLU DNNs of arbitrary depth and width using only $\mathcal O(P)$ parameters. This question holds significant practical implications, as training very deep yet narrow ReLU NNs is typically challenging. Additionally, wider neural networks often yield better generalization performance \cite{zagoruyko2016wide}.
Regrettably, even for $d=1$, this question is only partially answered, as highlighted in Lemma~\ref{lemm:interpolation-shen}. In this paper, we constructively demonstrate that ReLU DNNs of arbitrary depth and width with $\mathcal O(P)$ parameters can accurately represent any CPwL functions with $P$ segments on $[0,1]$. A similar finding for $d=1$ is also documented in \cite{daubechies2022nonlinear}, which employed a revised network architecture and complex construction. Furthermore, in this work, we also investigate the optimality of the representation and its application with KST for new approximation results.

Rather than representing any CPwL function, several interesting results in \cite{montufar2014number,telgarsky2015representation,mhaskar2016deep,lu2017expressive,mehrabi2018bounds} demonstrate that ReLU networks of sufficient depth, equipped with $\mathcal O(P)$ parameters, can produce specific CPwL functions that possess a number of segments surpassing any power of $P$. 
Among these findings, the hat basis function, denoted as $g(x) = 2{\rm ReLU}(x) - 4{\rm ReLU}(x-1/2) + 2{\rm ReLU}(x-1)$, plays a pivotal role. 
The initial observation that $g^i(x):= \overbrace{g\circ g \circ \cdots \circ g}^{i}(x)$ can be employed to approximate the square function $s(x) = x^2$ on the interval $[-1,1]$ with an exponential convergence rate was made in \cite{yarotsky2017error}. A more comprehensive and systematic analysis of this result from the hierarchical basis perspective is presented in \cite{he2022relu}. To summarize, \cite{he2022relu} demonstrated that the composition of $g(x)$ precisely generates the hierarchical basis employed for approximating the square function. Moreover, \cite{he2022relu} established that if the function can be expressed as the infinite series $\sum_{k=0}^\infty a_k g^k(x)$ for some $a_k \in \mathbb R$, it must be a quadratic function.

In addition, the expressive capabilities of neural networks utilizing ${\rm ReLU}^k$ activation functions have also been extensively investigated in works such as \cite{li2019better,xu2020finite,chen2022power}. These studies primarily concentrate on the representation of high-order splines within the interval $[0,1]$ or global polynomials on $\mathbb R^d$.

\paragraph{Approximation properties of ReLU NNs}
A considerable amount of research has been devoted to understanding the approximation properties of neural networks. An intuitive criterion to categorize these studies is based on the depth of the investigated neural networks. Exploring approximation properties for shallow (single-hidden-layer) and deep (more-than-one-hidden-layer) networks yields notably different insights.
From the 1990s, several studies, including \cite{hornik1989multilayer, cybenko1989approximation, jones1992simple, leshno1993multilayer, ellacott1994aspects, pinkus1999approximation}, have focused on examining the approximation properties of single hidden layer neural networks, mainly demonstrating their qualitative approximation properties.
In terms of quantitative error estimates, research in  \cite{barron1993universal,klusowski2018approximation, e2019priori, siegel2020approximation, e2021kolmogorov, siegel2022sharp, siegel2022high} has primarily focused on understanding and improving the approximation rate and defining and characterizing the underlying function spaces. 
These are now generally referred to as the Barron space~\cite{barron1993universal,e2019priori} or the variation space~\cite{siegel2022sharp,siegel2022high}, which is a generalization of the Barron space for ${\rm ReLU}^k$ activation functions from a different perspective.

Regarding the approximation properties of deep ReLU NNs, there are primarily two distinct streams of research. Both are based on significant findings about the expressivity of ReLU DNNs.
The first stream stems from the expressive power of ReLU DNNs for CPwL functions mentioned earlier and their combination with the bit-extraction technique \cite{bartlett1998almost,bartlett2019nearly} for deep ReLU NNs. 
For instance, \cite{yarotsky2018optimal} first accomplished an approximation rate of $\mathcal O(P^{-\frac{2}{d}})$ for Lipschitz continuous functions utilizing deep but constant-width ReLU NNs equipped with $\mathcal O(P)$ parameters. This achievement was initially facilitated by using a deep but narrow ReLU DNN to represent the CPwL interpolation of $f(\bm x)$ on $[0,1]^d$ with a uniform mesh, resulting in a rate of $\mathcal O(P^{-\frac{1}{d}})$. Subsequently, this rate was enhanced to $\mathcal O(P^{-\frac{2}{d}})$ through the application of the bit-extraction technique.
Additionally, by employing Lemma\ref{lemm:interpolation-shen} to form a projection mapping that shatters the domain $[0,1]^d$ into subcubes, and in combination with the bit-extraction technique, several approximation results were established in \cite{shen2019nonlinear,shen2020deep,lu2021deep,shen2022optimal} for general ReLU DNNs. For example, it was demonstrated that ReLU DNNs with $\mathcal O(L)$ hidden layers and $\mathcal O(N)$ neurons in each layer can achieve a rate of $\mathcal O(N^{-2/d}L^{-2/d})$ for any Lipschitz continuous function on $[0,1]^d$, extending the result in \cite{yarotsky2018optimal} to ReLU DNNs with arbitrary width.

The other stream is based on the earlier mentioned representation of the hierarchical basis approximation for the square function $s(x) = x^2$ on $[-1,1]$ using deep ReLU NNs. The core insight is that deep ReLU NNs can efficiently emulate the multiplication operation. Following this revelation, several research studies have been carried out to develop various exponential error bounds for classic or modified ReLU DNNs across different function families or measurements. For example, \cite{yarotsky2017error} established the first exponential approximation rate for a broad function class using ReLU DNNs with a limited width. Then, \cite{lu2017expressive} introduced a novel, more uniform network architecture to replicate the results in \cite{yarotsky2017error}. \cite{e2018exponential} further improved the network structure by incorporating an architecture resembling that of ResNet~\cite{he2016deep, he2016identity}, thus achieving an exponential convergence rate for a specific class of analytic function. \cite{montanelli2019new, montanelli2021deep, lu2021deep} investigated approximation properties on Koborov space, bandlimited functions, and $C^s$ functions by approximating sparse grids, truncated Chebyshev series, and local Taylor expansions. A series of results~\cite{opschoor2020deep, opschoor2019exponential, guhring2020error, marcati2023exponential} for different function spaces and norms have been acquired by utilizing approximation properties of finite element methods.

\paragraph{KST and its applications in DNNs}
The first iteration of the Kolmogorov Superposition Theorem (KST) was established in \cite{kolmogorov1957representation}, which involved $2d+1$ K-outer functions and $d(2d + 1)$ K-inner functions. Subsequent work aimed to decrease the count of inner and outer functions and to enhance their smoothness \cite{sprecher1965representation,sprecher1965structure,lorentz1966approximation,fridman1967improvement,sprecher1972improvement,sprecher1993universal,braun2009constructive,actor2017algorithm}. In this paper, we focus on a streamlined construction in \cite{lorentz1962metric,lorentz1966approximation}, comprising one K-outer function and $2d+1$ K-inner functions. It has been established that $2d+1$ is the lowest feasible count of inner functions \cite{hattori1993dimension}.
Regarding the continuity of the K-inner functions, \cite{fridman1967improvement} first demonstrated that the smoothness of K-inner functions in the original KST can be improved to be Lipschitz continuous. This smoothness level is optimal, as replacing the K-inner functions with continuously differentiable ones prevents the representation of some analytic functions via a KST-type formula \cite{vitushkin1964proof}. As noted in \cite{lorentz1966approximation}, the continuity of $\phi_k$ can be enhanced to become Lipschitz continuous with any Lipschitz constant within the range $(0,1)$. For a comprehensive historical overview of KST and a simplified proof for the two-dimensional case, readers are referred to \cite{morris2021hilbert}.

It's natural and intuitive to apply the KST to the study of neural networks, and this has been an active area of research in the field even before the advent of deep learning. One of the earliest attempts to utilize KST as a feed-forward neural network is seen in \cite{hecht1987kolmogorov}. Furthermore, the authors in \cite{igelnik2003kolmogorov} suggested an algorithm for a neural network that uses cubic spline functions to approximate both the inner and outer functions of the KST.
However, a significant challenge is that the K-outer function $g$ is dependent on $f$, and can vary greatly even if $f$ is smooth, as elucidated in \cite{girosi1989representation}. This implies that attempting to replicate the construction process of KST using DNNs could still result in the curse of dimensionality. For instance, the author in \cite{schmidt2021kolmogorov} utilized a version of KST from \cite{braun2009application,braun2009constructive}, approximated the K-inner functions with a deep ReLU neural network with $P$ total parameters, and achieved a rate of $\mathcal O(P^{-\frac{\alpha}{d}})$ for any $f(\bm x)$ on $[0,1]^d$ with $H^\alpha$ Hölder continuity.
After a detailed study of the constructions of K-outer functions as laid out in \cite{braun2009constructive}, the authors in \cite{montanelli2020error} defined a constrained function class that is characterized by bounded growth rates in constructing K-outer functions, consequently resulting in Lipschitz continuous K-outer functions. 
When dealing with target functions from this class, they utilized very deep ReLU NNs  with a constant width and $\mathcal O(P)$ total parameters to achieve an approximation rate of $\mathcal O\left(P^{-\frac{1}{\log(d)}}\right)$.
Recently, by directly assuming that the K-outer function in~\cite{lorentz1966approximation} is Lipschitz continuous and restricting the target function to this class, the authors in \cite{lai2021kolmogorov} demonstrated a rate of $\mathcal O\left(P^{-\frac{1}{2}}\right)$ for ReLU DNNs with two hidden layers and $\mathcal O(P)$ parameters.

\section{Notation and preliminary}
\label{sec:notation}
In this section, we introduce some notations and preliminary results which will be used in the remaining context.

Here, we first define a ReLU DNN function $f(\bm x)$ with $L$ hidden layers on $\mathbb R^d$ as
\begin{equation*}
    \begin{cases}
        f^0(\bm x) &=  \bm x,\\
        f^\ell (\bm x) &= \sigma\left(W^\ell f^{\ell-1}(\bm x) + b^\ell \right), \quad \ell=1:L\\
        f(\bm x) &= W^{L+1} f^L(\bm x) + b^{L+1},
    \end{cases}
\end{equation*}
where 
\begin{equation*}
    \sigma(x) = {\rm ReLU}(x) := \max\{0,x\},
\end{equation*}
and
\begin{equation*}
    W^\ell: \mathbb R^{N_{\ell-1}} \mapsto \mathbb R^{N_\ell}
\end{equation*}
with $N_0 = \mathbb R^d$ and $N_{L+1} = 1$.

\begin{notation}\label{not:C}
    We first introduce some notations about continuous functions.
    \begin{itemize}
    \item For continuous function $f: [a,d] \mapsto \mathbb R$ , we introduce the modulus of continuity of $f$ defined via
\begin{equation*}
    \omega_g(r) := \sup_{x,y\in[a,b]} \left\{ \left|f(x) - f(y)\right|:  |x-y| \le r  \right\},
\end{equation*}
for any $r>0$.
\item For Lipschitz continuous functions, we denote
\begin{equation*}
    {\rm Lip}_C(\Omega) := \left\{ f: |f(\bm x)-f(\bm y)| \le C\|\bm x-\bm y\|, ~ \bm x, \bm y\in \Omega\right\}
\end{equation*}
as the function class defined on $\Omega \subset \mathbb R^d$ with Lipschitz constant less than $C$.
\item For any continuous function $f(\bm x)$ on $\Omega \subset \mathbb R^d$, we define the $L^{\infty}$ norm as
\begin{equation*}
    \|f\|_{L^{\infty}(\Omega)} = \sup_{\bm x\in \Omega} \left|f(x)\right|.
\end{equation*}
    \end{itemize}
\end{notation}

\begin{notation}\label{not:N}
We then define the following notations about function classes of ReLU DNNs, CPwL functions, and CPwL interpolations.
\begin{itemize}
\item For ReLU DNNs with $L$ hidden layers and $N_\ell$ neurons at $\ell$-th hidden layer, we define the following notation
    \begin{equation*}
        \mathcal N_{(N_1, \cdots, N_L)} = \mathcal N_{N_{1:L}}.
    \end{equation*}
    In particular, if the $N_\ell = N$ for all $\ell = 1:L$, we denote 
    \begin{equation*}
        \mathcal N_{(N_1, \cdots, N_L)} = \mathcal N_N^L,
    \end{equation*}
    and call $N$ as the width of the network.
    
\item For the set of piecewise linear functions on $[0,1]$ with at most $N$ segments, which is also known as the adaptive linear finite element space with at most $N$ elements~\cite{devore1998nonlinear}, we denote it as $\mathcal A_N$. More precisely, for any $f(x) \in \mathcal A_N$, there exists grid points 
\begin{equation*}
    \left\{ 0 = x_0 \le x_1 \le \cdots < x_{N-1} \le x_N = 1\right\},
\end{equation*}
such that 
\begin{equation*}
    \left. f(x) \right|_{[x_{i-1}, x_{i}]} ~~\text{is a linear function}\quad \forall i = 1:N,
\end{equation*}
if $x_{i-1} \neq x_{i}$.
\item For any grid points on $[a,b]$ with $N$ segments, we define them in a set as
\begin{equation*}
    \mathcal T =\left\{a = x_0 < x_1<\cdots < x_{N-1} < x_N=b \right\}.
\end{equation*}
Then, $\mathcal I_{\mathcal T}(f)$ denotes the CPwL interpolation of the continuous $f(x)$ on the grid  $\mathcal T$ which is given by
\begin{equation*}
	\mathcal I_{\mathcal T}(f)(x_i) = f(x_i) \quad \forall i=0:N 
\end{equation*}
and 
\begin{equation*}
    \left. \mathcal I_{\mathcal T}(f) \right|_{[x_{i-1}, x_{i}]} \text{ is a linear function} \quad \forall i=1:N.
\end{equation*}

\item More precisely, for any $x\in [0,1]$, we have
\begin{equation*}
	\mathcal I_{\mathcal T}(f)(x) = \sum_{i=0}^{N^2} f(x_i) b_i(x),
\end{equation*}
where
\begin{equation}\label{eq:bi}
	b_i(x) = \begin{cases}
		\frac{x - x_{i-1}}{x_i - x_{i-1}}, \quad &x \in [x_{i-1}, x_{i}], \\
		\frac{x - x_{i+1}}{x_i - x_{i+1}}, \quad &x \in [x_{i}, x_{i+1}], \\
		0, \quad &\text{others},
	\end{cases}
\end{equation}
is the nodal basis function at the node $x_i$. 
\end{itemize}
\end{notation}
 
\begin{properties}\label{pro:interpolation}
In the following, we introduce some basic properties of the CPwL interpolation function $\mathcal I_{\mathcal T}(f)$ and its error estimates.
\begin{itemize}
    \item If $f: [a,b] \mapsto \mathbb R$ is a continuous function, we have
    \begin{equation*}
        \|f - \mathcal I_{\mathcal T}(f)\|_{L^{\infty}([a,b])} \le 2\omega_f(r),
    \end{equation*}
    where $\mathcal T =\left\{a = x_0 < x_1<\cdots < x_{N-1} < x_N=b \right\}$ and $r = \max_{i} |x_{i} - x_{i-1}|$.
    \item If $f \in {\rm Lip}_C([a,b])$, we have
    \begin{equation*}
        \|f - \mathcal I_{\mathcal T}(f)\|_{L^{\infty}([a,b])} \le C\max_{i=1:N}|x_i - x_{i-1}|,
    \end{equation*}
    where $\mathcal T =\left\{a = x_0 < x_1<\cdots < x_{N-1} < x_N=b \right\}$.
    \item If $f: [a,b] \mapsto \mathbb R$ is a continuous function, we have
    \begin{equation*}
         \left\|\mathcal I_{\mathcal T}(f)\right\|_{L^\infty([a,b])} \le \|f\|_{L^{\infty}([a,b])},
    \end{equation*}
    for any  $\mathcal T =\left\{a = x_0 < x_1<\cdots < x_{N-1} < x_N=b \right\}$ and any $N$.
\end{itemize}
\end{properties}

\section{Expressive power of ReLU DNNs on $[0,1]$}
\label{sec:expres}
In this section, we prove that $\mathcal A_{P} \subset \mathcal N_{N}^L$ if $P = \mathcal O(N^2L)$, thereby providing a positive answer to Question~\ref{que:NL}. Our proof unfolds in three distinct stages.
We initiate by defining a ReLU neural network with a singular hidden layer that operates over a distinct subgrid $\mathbf N$ and subsequently delve into its inherent properties.
Using the insights from the specially designed one-hidden-layer ReLU NN, we provide a constructive proof to address Question~\ref{que:N2}, but specifically for the scenario when 
$L=2$.
In the final step, we generalize our findings to accommodate ReLU DNNs of any depth $L$ by reinterpreting ReLU DNNs with two hidden layers and a particular structure into ReLU DNNs with generic depth and width.

\subsection{A special class of ReLU NNs with one hidden layer}
For any grid points $\mathcal T:= \left\{ 0 = x_0 < x_1 < \cdots < x_{N^2-1} < x_{N^2} = 1 \right\}$ on $[0,1]$,
let us denote the specific subset of $\mathcal T$ as
\begin{equation*}
		\mathbf N := \{ x_{k} ~|~ k = jN-1, jN, \text{ or } jN+1 \}.
\end{equation*}
Then, we define the following set of one-hidden-layer ReLU NNs
\begin{equation*}
	\Sigma_{\mathbf N} :=\left\{ \phi(x) = \sum_{j=1}^{3N+1} w_j \sigma(x - \tilde x_j) + b ~ :~ w_j, b\in \mathbb{R} \right\},
\end{equation*}
where $\{ \widetilde x_j\}_{j=1}^{3N+1} = \mathbf N$ and $\widetilde x_j < \widetilde x_{j+1}$ follows the order in $\mathbf N$.
More precisely, we have
\begin{equation}\label{eq:def_xtilede}
\begin{split}
\widetilde x_1 &= x_0, ~ \widetilde x_2 = x_1, \\
&\vdots \\
\widetilde x_{3j} &= x_{jN-1}, ~ \widetilde x_{3j+1} = x_{jN}, ~ \widetilde x_{3j+2} = x_{jN+1}, \\
&\vdots \\
\widetilde x_{3N} &= x_{N^2-1}, ~ \widetilde x_{3N+1} = x_{N^2}.
\end{split}
\end{equation}

We first notice the following lemma about the expressive power of $\Sigma_{\mathbf N}$ on the grid produced by $\mathbf N$.
\begin{lemma}[\cite{he2020relu,shen2019nonlinear}]\label{lemm:phi0}
For arbitrary $y_j \in \mathbb{R}$ with $j=1:3N+1$, 
we can find $\phi \in 	\Sigma_{\mathbf N} $ such that
\begin{equation*}
	\phi(\widetilde{x}_j ) = y_j, \quad \forall ~ \widetilde{x}_j \in \mathbf N,
\end{equation*}
and $ \phi(x) $ is linear on ${[\widetilde x_{j}, \widetilde x_{j+1}]}$ for $j=1:3N$.	
That is, $\Sigma_{\mathbf N}$ can represent any CPwL interpolation functions on the grid $\mathbf N$.
\end{lemma}

\begin{proposition}\label{prop:phi}
	Furthermore, we have the following properties of $\Sigma_{\mathbf N}$.
	\begin{enumerate}
		\item $\alpha \phi_1 + \beta \phi_2 \in \Sigma_{\mathbf N}$, for any $\phi_1, \phi_2 \in \Sigma_{\mathbf N}$ and $\alpha, \beta \in \mathbb{R}$.
		\item For any $2:N-1$ and $y_j > 0$ with $j=1:N$, there exists an unique $\phi(x) \in \Sigma_{\mathbf N}$ such that
		\begin{equation}\label{eq:phi-}
			\begin{cases}
				\phi(x_{jN}) = 0, \quad &\text{ for } j = 0:N, \\
				\phi(x_{(j-1)N+k-1}) = 0, \quad &\text{ for } j = 1:N, \\
				\phi(x_{(j-1)N+k}) = y_j, \quad &\text{ for } j = 1:N.
			\end{cases}
		\end{equation}
  In particular, for $k=N-1$, we have 
  \begin{equation*}
	\phi(x_{(j-1)N+1}) < 0, ~ \forall j = 1:N.
 \end{equation*}
 For $k=2:N-2$, we have
	\begin{equation*}
\phi(x_{(j-1)N+1}) < 0 ~ \text{ and } ~ \phi(x_{jN-1}) > 0, ~ \forall j = 1:N.
 	\end{equation*}
	\item For any $2:N-2$ and $y_j > 0$ with $j=1:N$, there exists an unique $\phi(x) \in \Sigma_{\mathbf N}$ such that
	\begin{equation}\label{eq:phi+}
		\begin{cases}
			\phi(x_{jN}) = 0, \quad &\text{ for } j = 0:N, \\
			\phi(x_{(j-1)N+k}) = y_j, \quad &\text{ for } j = 1:N, \\
			\phi(x_{(j-1)N+k+1}) = 0, \quad &\text{ for } j = 1:N.
		\end{cases}
	\end{equation}
	In particular, for $k=1$, we have 
         \begin{equation*}
	   \phi(x_{jN-1}) < 0, ~ \forall j = 1:N.
        \end{equation*}
        For $k=2:N-2$, we have
	\begin{equation*}
		\phi(x_{(j-1)N+1}) > 0 ~ \text{ and } ~ \phi(x_{jN-1}) < 0, ~ \forall j = 1:N.
	\end{equation*}
	\end{enumerate}
\end{proposition}

\begin{proof}
The first proposition is true because of the definition of $\Sigma_{\mathbf N}$. 

For the second proposition, let us first prove it when $k=N-1$. For better understanding, one may refer to the diagram of $\phi_{N-1}(x)$ in Fig.~\ref{fig:phi01N-1}.
For this case, we only need to determine the value of $\phi(x_{(j-1)N+1})$ for all $j=1:N$.  First, we denote $\ell^{N-1}_j(x)$ as the linear function on $[x_{(j-1)N+1}, x_{jN-1}]$ which passes $(x_{jN-2}, 0)$ and $(x_{jN-1}, y_j )$ for all $j=1:N$.  Then, we define $\phi(x)$ by taking 
\begin{equation*}
	\begin{cases}
		\phi(x_{jN - 1}) = y_j, \quad &\text{ for } j = 1:N, \\
		\phi(x_{jN}) = 0, \quad &\text{ for } j = 0:N, \\
		\phi(x_{(j-1))N + 1}) =  \ell^{N-1}_j(x_{(j-1)N+1}), \quad &\text{ for } j = 1:N.
	\end{cases}
\end{equation*}
Therefore, we have
\begin{equation*}
	\phi(x) = \ell^{N-1}_j(x), \quad \forall x \in [x_{(j-1)N+1}, x_{jN-1}]
\end{equation*}
for all $j=1:N$. 
Then, we have
\begin{equation*}
	\phi(x_{jN-2}) =  \ell^{N-1}_j(x_{jN-2}) = 0
\end{equation*}
and 
\begin{equation*}
	\phi(x_{(j-1)N+1}) =  \ell^{N-1}_j(x_{(j-1)N+1}) < 0,
\end{equation*}
since $\ell^{N-1}_j(x_{jN-1}) =  y_j >0$. This finishes the proof for $k=N-1$.

For the second proposition with general $k=2:N-2$, one may refer to the diagram of $\phi_k^-(x)$ in Fig.~\ref{fig:phi_k-+}. 
For this case, we need to determine $\phi(x_{jN - 1})$ and $\phi(x_{(j-1))N + 1})$ for all $j=1:N$. 
First, we denote $\ell^{k-}_j(x)$ as the function on $[x_{(j-1)N+1}, x_{jN-1}]$ which passes $(x_{(j-1)N+k-1}, 0)$ and $(x_{(j-1)N+k}, y_j )$ for all $j=1:N$.  Then, we define $\phi(x)$ by taking 
\begin{equation*}
	\begin{cases}
		\phi(x_{jN - 1}) = \ell^{k-}_j(x_{jN-1}), \quad &\text{ for } j = 1:N, \\
		\phi(x_{jN}) = 0, \quad &\text{ for } j = 0:N, \\
		\phi(x_{(j-1))N + 1}) =  \ell^{k-}_j(x_{(j-1)N+1}), \quad &\text{ for } j = 1:N.
	\end{cases}
\end{equation*}
Thus, we have
\begin{equation*}
	\phi(x) = \ell^{k-}_j(x), \quad \forall x \in [x_{(j-1)N+1}, x_{jN-1}]
\end{equation*}
for all $j=1:N$. So, it follows that
\begin{equation*}
	\begin{cases}
	\phi(x_{(j-1)N+k-1}) =  \ell^{N-1}_j(x_{(j-1)N+k-1}) = 0, \\
	\phi(x_{(j-1)N+k}) = \ell^{N-1}_j(x_{(j-1)N+k}) =  y_j.
	\end{cases}
\end{equation*}
In addition, we have
$$
\phi(x_{(j-1)N+1}) =  \ell^{N-1}_j(x_{(j-1)N+1}) < 0
$$
and
$$
\phi(x_{jN-1}) = \ell^{N-1}_j(x_{jN-1}) > 0
$$
for all $j=1:N$, since $ \ell^{N-1}_j((j-1)N+k) =  y_j >0$ . This finishes the proof for $k=2:N-2$.

For the third proposition, we can prove it in a similar way by constructing $\ell^{1}_j(x)$ and $\ell^{k+}_j(x)$ for $k=2:N-2$. One may also refer to the diagram of $\phi_1(x)$ in Fig.~\ref{fig:phi01N-1} and
the diagram of $\phi_k^{+}(x)$ in Fig.~\ref{fig:phi_k-+} in Appendix~\ref{sec:appendix}.
\end{proof}

\subsection{Expressive power of ReLU NNs with two hidden layers}
Having established these preliminary results $\Sigma_{\mathbf N}$, we now present our main theorem as follows.
\begin{theorem}\label{thm:interpolation}For any continuous function $f(x)$ and grid points $\mathcal T: = \{ 0 = x_0 < x_1 < \cdots < x_{N^2-1} < x_{N^2} = 1\}$ on $[0,1]$, there exists a ReLU DNN function with two hidden layers, specifically $\widetilde f(x) \in \mathcal N_{(3N+1,3(N-2))}$, such that $\widetilde f = \mathcal I_{\mathcal T}(f)$ on $[0,1]$.
\end{theorem}

\begin{proof}
Since $\widetilde f(x) \in \mathcal N_{(3N+1,3(N-2))}$, we can write it as
\begin{equation}\label{eq:tildef}
	\widetilde f(x) = \sum^{3(N-2)}_{i=1} W^{3}_i \sigma \left( \sum_{j=1}^{3N+1}W^{2}_{ij} \sigma \left( W^{1}_j x + b^{1}_j \right) + b^{2}_i \right) + b^{3}.
	\end{equation}
Without loss of generality, let us assume that $f(x) > 0$ for all $x\in [0,1]$. Otherwise, we can choose $b^{3} = 1- \inf_{x\in[0,1]} f(x)$ in \eqref{eq:tildef}.  Then, we split the proof into the following steps.

\paragraph{Construction of the first layer}
First, let $W^{1}_j =1 $ and $b^{1}_j = -\widetilde x_j$ where $\widetilde{x}_j$ are defined in \eqref{eq:def_xtilede}. 
Then, for each $1\le i \le 3(N-2)$, the input of the $i$-th neuron for the second hidden layer is
\begin{equation*}
\phi_i(x) = W^{2}_{ij} \sigma \left( W^{1}_j x + b^{1}_j \right) + b^{2}_i \in \Sigma_{\mathbf N}.
\end{equation*}
Given the Proposition~\ref{prop:phi}, let us show how to construct these $\phi_i(x) \in \Sigma_{\mathbf N} $ for $i = 1:3(N-2)$. For simplicity, we adopt the following notation 
\begin{equation*}
	\phi_0(x), \phi_1(x), \cdots, \phi^{-}_k(x), \phi_k(x), \phi^+_k(x), \cdots, \phi_{N-1}(x),
\end{equation*}
for $k=2:N-2$ to replace $\{\phi_i(x)\}_{i=1}^{3(N-2)}$.

\paragraph{Construction of $\phi_0(x), \phi_1(x)$, and $\phi_{N-1}(x)$}
From Lemma~\ref{lemm:phi0}, we can find $\phi_0(x) \in \Sigma_{\mathbf N}$ such that
\begin{equation}\label{eq:phi0}
\begin{cases}
\phi_{0}(x_{jN - 1}) = 0, \quad &\text{ for } j = 1:N, \\
\phi_{0}(x_{jN}) = f(x_{jN}), \quad &\text{ for } j = 0:N, \\
\phi_{0}(x_{(j-1))N + 1}) = 0, \quad &\text{ for } j = 1:N.
\end{cases}
\end{equation}
Namely, we have
\begin{equation*}
	\phi_0(x) = \sum_{j=0}^N f(x_{jN}) b_{jN}(x),
\end{equation*}
where $\phi_{i}(x)$ is the nodal basis function at $x_i$ regarding the mesh $\mathcal T$ as defined in \eqref{eq:bi} in Notation~\ref{not:N}. 

According to the third proposition in Proposition~\ref{prop:phi} with $k=1$, we can find 
$\phi_1(x) \in \Sigma_{\mathbf N}$ such that
\begin{equation}\label{eq:phi1}
	\begin{cases}
		\phi_{1}(x_{jN}) = 0, \quad &\text{ for } j = 0:N, \\
		\phi_{1}(x_{(j-1)N+1}) = f(x_{(j-1)N+1}), \quad &\text{ for } j = 1:N, \\
		\phi_{1}(x_{(j-1)N+2}) = 0, \quad &\text{ for } j = 1:N,
	\end{cases}
\end{equation}
and
\begin{equation}\label{eq:phi1_ineq}
	\phi_{1}(x_{jN-1}) < 0, ~\forall j=1:N.
\end{equation}
According to the second proposition in Proposition~\ref{prop:phi} with $k=N-1$, we can find 
$\phi_{N-1}(x) \in \Sigma_{\mathbf N}$ such that
\begin{equation}\label{eq:phiN-1}
	\begin{cases}
		\phi_{N-1}(x_{jN-2}) = 0, \quad &\text{ for } j = 1:N, \\
		\phi_{N-1}(x_{jN-1}) = f(x_{jN-1}), \quad &\text{ for } j = 1:N, \\
		\phi_{N-1}(x_{jN}) = 0, \quad &\text{ for } j = 0:N,
	\end{cases}
\end{equation}
and
\begin{equation}\label{eq:phiN-1_ineq}
	\phi_{1}(x_{jN+1}) < 0, ~\forall j=0:N-1.
\end{equation}
One may refer to the diagrams in Fig.~\ref{fig:phi01N-1} for $\phi_0(x), \phi_1(x)$, and $\phi_{N-1}(x)$.

Now, we introduce the following construction of $\phi^{-}_k(x), \phi^{+}_k(x)$, and $\phi_k(x)$ for $k=2:N-2$.

\paragraph{Construction of $\phi^{-}_k(x), \phi^{+}_k(x)$, and $\phi_k(x)$ for $k=2:N-2$}
Given the second proposition in Proposition~\ref{prop:phi} for $2 \le k \le N-2$, we can find unique $\phi_{k}^-(x) \in \Sigma_{\mathbf N}$ such that
\begin{equation}\label{eq:phik-}
	\begin{cases}
		\phi^-_{k}(x_{jN}) = 0, \quad &\text{ for } j = 0:N, \\
		\phi^-_{k}(x_{(j-1)N+k-1}) = 0, \quad &\text{ for } j = 1:N, \\
		\phi^-_{k}(x_{(j-1)N+k}) = f(x_{(j-1)N+k}), \quad &\text{ for } j = 1:N
	\end{cases}
\end{equation}
and
\begin{equation}\label{eq:phik-_ineq}
	\phi^-_k(x_{(j-1)N+1}) < 0 ~ \text{ and } ~ \phi^-_k(x_{jN-1}) > 0 ~ \forall j =1:N.
\end{equation}
Correspondingly, we can find unique $\phi_{k}^+(x) \in \Sigma_{\mathbf N}$ based on the third proposition in Proposition~\ref{prop:phi} for $2 \le k \le N-2$ such that
\begin{equation}\label{eq:phik+}
	\begin{cases}
		\phi^+_{k}(x_{jN}) = 0, \quad &\text{ for } j = 0:N, \\
		\phi^+_{k}(x_{(j-1)N+k}) = f(x_{(j-1)N+k}), \quad &\text{ for } j = 1:N, \\
		\phi^+_{k}(x_{(j-1)N+k+1}) = 0, \quad &\text{ for } j = 1:N,
	\end{cases}
\end{equation}
and
\begin{equation}\label{eq:phik+_ineq}
	\phi^+_k(x_{(j-1)N+1}) > 0  \text{ and } \phi^+_k(x_{jN-1}) < 0, ~ \forall j = 1:N.
\end{equation}
Finally, we construct 
\begin{equation}\label{eq:phik}
	\phi_k = \phi_k^+ - \phi_k^- \in \Sigma_{\mathbf N},
\end{equation}
based on the first proposition in Proposition~\ref{prop:phi}.

\paragraph{Construction of $\widetilde f(x)$}
We claim that the following construction 
\begin{equation*}
	\widetilde f(x) = \sum_{k=0}^{N-1} \varphi _k(x)
\end{equation*}
satisfies the requirements in Theorem~\ref{thm:interpolation}, i.e, $\widetilde f = \mathcal I_{\mathcal T}(f)$, where
\begin{equation*}
    \varphi _0 = \sigma(\phi_0), \quad \varphi _1 = \sigma(\phi_1), \quad \text{and} \quad \varphi _{N-1} = \sigma(\phi_{N-1})
\end{equation*}
and
\begin{equation*}
	\varphi_k  = \sigma(\phi_k^+) -\sigma(\phi_k) +\sigma(-\phi_k^-), \quad k=2:N-2.
\end{equation*}
We claim that
\begin{equation}\label{eq:varphik_interp}
	\varphi_{k}(x) = \sum_{j=0}^{N-1} f(x_{jN+k}) b_{jN+k}(x)
\end{equation}
for all $k=0:N-1$. For a graphical representation and better insight into $\varphi_{k}(x)$, readers are directed to Fig.\ref{fig:varphi_k} in Appendix\ref{sec:appendix}. The proof will be complete once we validate the identity in~\eqref{eq:varphik_interp} for all $k=0:N-1$.

\paragraph{Interpretation and properties of $\varphi_k(x)$ for $k=0, 1$, and $N-1$}
For $\varphi _0$, we have
\begin{equation*}
\begin{split}
\varphi _0(x) &= \sigma(\phi_0(x)) = \sigma\left( \sum_{j=0}^N f(x_{jN})b_{jN}(x)\right) \\
&= \sum_{j=0}^N f(x_{jN})b_{jN}(x), ~~\forall x \in [0,1],
\end{split}
\end{equation*}
since $f(x_{jN}) > 0$ and $b_{jN}(x) \ge 0$ for any $j=0:N$ on $[0,1]$.

We then check $\varphi_1(x)$. According to Lemma~\ref{lemm:phi0} the properties of $\phi_1(x)$ in \eqref{eq:phi1} and \eqref{eq:phi1_ineq}, we have
\begin{equation*}
\phi_{1}(x) = f(x_{jN+1})b_{jN+1}(x), \quad \forall x\in [x_{jN}, x_{jN+2}]
\end{equation*}
and
\begin{equation*}
	\phi_{1}(x) \le 0, \quad \forall x\in [x_{jN+2}, x_{(j+1)N}]
\end{equation*}
for $j=0:N-1$. This shows that
\begin{equation*}
\begin{split}
	\varphi_{1}(x) &= \sigma\left( \phi_{1}(x) \right) = \phi_{1}(x) \\
 &= f(x_{jN+1})b_{jN+1}(x), \quad \forall x \in [x_{jN}, x_{jN+2}]
\end{split}
\end{equation*}
and
\begin{equation*}
	\varphi_{1}(x) = \sigma\left(\phi_1(x) \right) =  0, \quad \forall x\in [x_{jN+2}, x_{(j+1)N}]
\end{equation*}
for $j=0:N-1$, which completes that proof of \eqref{eq:varphik_interp} for $k=1$. 
In addition, $\varphi_{N-1}(x)$ can also be done in a similar way to $\varphi_1(x)$. 

\paragraph{Interpretation and Properties of $\varphi_k(x)$ for $k=2:N-2$}
Now, we only need to check \eqref{eq:varphik_interp} for $k=2:N-2$. 
For these cases, we first have the following properties of $\phi_k^-(x)$, $\phi_k^+(x)$, and $\phi^+_k - \phi_k^-$.
\begin{itemize}
	\item For $\phi_k^-(x)$, we have 
	\begin{enumerate}
		\item $\phi_k^-(x)$ is linear and $\phi_k^-(x)\le 0$ on $[x_{jN}, x_{jN+1}]$ for $j=0:N-1$,
		\item $\phi_k^-(x)$ is linear on $[x_{jN+1}, x_{(j+1)N-1}]$, where
		$\phi_k^-(x) \le 0$ on $[x_{jN+1}, x_{jN+k-1}]$ and $ \phi_k^-(x) \ge 0$ on $[x_{jN+k-1}, x_{(j+1)N-1}]$ for all $j=0:N-1$, 
		\item $\phi_k^-(x)$ is linear and $\phi_k^-(x)\ge 0$ on $[x_{jN-1}, x_{jN}]$ for $j=1:N$.
	\end{enumerate}
	\item For $\phi_k^+(x)$, we have 
	\begin{enumerate}
		\item $\phi_k^+(x)$ is linear and $\phi_k^+(x)\ge 0$ on $[x_{jN}, x_{jN+1}]$ for $j=0:N-1$,
		\item $\phi_k^+(x)$ is linear on $[x_{jN+1}, x_{(j+1)N-1}]$, where
		$\phi_k^+(x) \ge 0$ on $[x_{jN+1}, x_{jN+k+1}]$ and $ \phi_k^+(x) \le 0$ on $[x_{jN+k+1}, x_{(j+1)N-1}]$ for all $j=0:N-1$, 
		\item $\phi_k^+(x)$ is linear and $\phi_k^+(x)\le 0$ on $[x_{jN-1}, x_{jN}]$ for $j=1:N$.
	\end{enumerate}	
	\item For $\phi_k(x) = \phi_k^+(x) - \phi_k^-(x)$, we have 
	\begin{enumerate}
		\item $\phi_k(x)$ is linear and $\phi_k(x) \ge 0$ on $[x_{jN}, x_{jN+1}]$ for $j=0:N-1$,
		\item $\phi_k(x)$ is linear on $[x_{jN+1}, x_{(j+1)N-1}]$, where
		$\phi_k(x) \ge 0$ on $[x_{jN+1}, x_{jN+k}]$ and $\phi_k(x) \le 0$ on $[x_{jN+k}, x_{(j+1)N-1}]$ for all $j=0:N-1$, 
		\item $\phi_k(x)$ is linear and $\phi_k(x)\le 0$ on $[x_{jN-1}, x_{jN}]$ for $j=1:N$.
	\end{enumerate}	
\end{itemize}
Because of the definition of activation function $\sigma(x) = {\rm ReLU}(x) := \max \{ 0, x\}$, we know that $\varphi_k(x) =  \sigma(\phi_k^+) -\sigma(\phi_k) +\sigma(-\phi_k^-)$
is a piecewise linear function on $[x_{jN}, x_{jN+1}]$, $[x_{jN+1}, x_{jN+k-1}]$, $[x_{jN+k-1}, x_{jN+k}]$, $[x_{jN+k}, x_{jN+k+1}]$,  $[x_{jN+k+1}, x_{(j+1)N-1}]$, and $[x_{(j+1)N-1}, x_{(j+1)N}]$ for all $j=0:N-1$.
That is, we can determine $\varphi_k(x)$ by evaluating it at $x_{jN}$ for all $j=0:N$, and $x_{jN+1}$, $x_{jN+k-1}$, $x_{jN+k}$, $x_{jN+k+1}$, $x_{(j+1)N-1}$ for all $j=0:N-1$.
Then, we have
\begin{itemize}
	\item $\varphi_k(x_{jN}) = \sigma\left(\phi_k^+(x_{jN})\right) -\sigma\left(\phi_k(x_{jN})\right) +\sigma\left(-\phi_k^-(x_{jN})\right) = 0-0+0 = 0$ for all $j=0:N$, since $\phi_k^+(x_{jN}) = \phi_k^-(x_{jN}) = 0$,
	\item $\varphi_k(x_{jN+1}) = \sigma\left(\phi_k^+(x_{jN+1})\right) -\sigma\left(\phi_k(x_{jN+1})\right) +\sigma\left(-\phi_k^-(x_{jN+1})\right) = \phi_k^+(x_{jN+1})-\phi_k(x_{jN+1})+\left(-\phi_k^-(x_{jN+1})\right) = 0$ for all $j=0:N-1$, since $\phi_k^+(x_{jN+1})\ge 0$ and $\phi_k^-(x_{jN+1})\le 0$,
	\item $\varphi_k(x_{jN+k-1}) = \sigma\left(\phi_k^+(x_{jN+k-1})\right) -\sigma\left(\phi_k(x_{jN+k-1})\right) +\sigma\left(-\phi_k^-(x_{jN+k-1})\right) = \phi_k^+(x_{jN+k-1})-\phi_k(x_{jN+k-1})+0 = 0$ for all $j=0:N-1$, since $\phi_k^+(x_{jN+k-1})\ge 0$ and $\phi_k^-(x_{jN+k-1})= 0$,
	\item $\varphi_k(x_{jN+k}) = \sigma\left(\phi_k^+(x_{jN+k})\right) -\sigma\left(\phi_k(x_{jN+k})\right) +\sigma\left(-\phi_k^-(x_{jN+k})\right) = \phi_k^+(x_{jN+k-1})-0+0 = f(x_{jN+k})$ for all $j=0:N-1$, since $\phi_k^+(x_{jN+k}) = \phi_k^-(x_{jN+k}) = f(x_{jN+k}) > 0$,
	\item $\varphi_k(x_{jN+k+1}) = \sigma\left(\phi_k^+(x_{jN+k+1})\right) -\sigma\left(\phi_k(x_{jN+k+1})\right) +\sigma\left(-\phi_k^-(x_{jN+k+1})\right) = 0-\left( - \phi_k^-(x_{jN+k+1})\right)+\left( - \phi_k^-(x_{jN+k+1})\right) = 0$ for all $j=0:N-1$, since $\phi_k^+(x_{jN+k+1}) = 0$ and $\phi_k^-(x_{jN+k+1}) \le 0$,
	\item $\varphi_k(x_{(j+1)N-1}) = \sigma\left(\phi_k^+(x_{(j+1)N-1})\right) -\sigma\left(\phi_k(x_{(j+1)N-1})\right) +\sigma\left(-\phi_k^-(x_{(j+1)N-1})\right) = 0-0+0= 0$ for all $j=0:N-1$, since $\phi_k^+(x_{(j+1)N-1}) \le 0$ and $\phi_k^-(x_{(j+1)N-1}) \ge 0$.
\end{itemize}
In summary, we have $\varphi_k(x_{jN+k}) = f(x_{jN+k})$ and $\varphi_k(x) = 0$ on $[x_{jN}, x_{jN+1}]$, $[x_{jN+1}, x_{jN+k-1}]$, $[x_{jN+k+1}, x_{(j+1)N-1}]$, and $[x_{(j+1)N-1}, x_{(j+1)N}]$ for all $j=0:N-1$.
This shows that 
\begin{equation*}
	\varphi_k(x) = \sum_{j=0}^{N-1} f(x_{jN+k}) b_{jN+k}(x),
\end{equation*}
for all $k=2:N-2$. In the end, we complete the proof.
\end{proof}

\subsection{From two hidden layers to $L$ hidden layers}
In short, Theorem~\ref{thm:interpolation} shows that
\begin{equation*}
  \mathcal  A_{N^2} \subset \mathcal N_{(3N+1, 3(N-2)} \subset \mathcal N_{3N+1, 3N}.
\end{equation*}
Since the index $k$ in the construction of $\varphi_{k}$ by using $\phi_k^-$, $\phi_k^+$, and $\phi_k$ can be arbitrary, we can easily generalize the original version of Theorem~\ref{thm:interpolation} to ReLU DNNs with two hidden layers but different width per layer.
\begin{corollary}\label{cor:AN2}
    On $[0,1]$, we have
    \begin{equation*}
      \mathcal  A_{NM} \subset \mathcal N_{(3N+1, 3(M-2)} \subset \mathcal N_{(3N+1, 3M)}.
    \end{equation*}
\end{corollary}

To generalize the expressive theorem of ReLU DNNs from two hidden layers to general depth, we then note the following lemma~\cite{shen2020deep}.
\begin{lemma}[Lemma 3.2 of \cite{shen2020deep}]\label{lem:2toL}
    For any ReLU DNNs on $\mathbb R^d$, we have
    \begin{equation*}
        \mathcal N_{(N,NL)} \subset \mathcal N_{2N+2}^{L+1},
    \end{equation*}
    for any $d$ and $L \ge 1$.
\end{lemma}

As a result, we have the following theorem about the expressive power of ReLU DNNs on $[0,1]$.
\begin{theorem}
     On $[0,1]$, we have
    \begin{equation*}
        \mathcal A_{N^2L}  \subset \mathcal N_{6N+4}^{L+1},
    \end{equation*}
    for any $L\ge 1$.
\end{theorem}
\begin{proof}
    By using Corollary~\ref{cor:AN2} and Lemma~\ref{lem:2toL}, we have
      \begin{equation*}
      \begin{split}
        \mathcal A_{N^2L} &= \mathcal A_{N\times NL} \subset \mathcal N_{(3N+1, 3NL)} \\
        &\subset \mathcal N_{(3N+1, (3N+1)L)} \subset \mathcal N_{6N+4}^{L+1},
      \end{split}
    \end{equation*}  
    for any $L \ge 1$.
\end{proof}

\section{The optimality of the expressivity}
\label{sec:optimalrep}
In this section, we demonstrate that $\mathcal O(N)$ parameters are requisite for representing any CPwL functions defined on $[0,1]$ with $N$ segments, which answers the optimality in Question~\ref{que:NL}. Central to our argument is an examination of the shattering capabilities of ReLU DNNs.

Before we state our main result, let us first introduce the definition of shattering sets \cite{sauer1972density,shelah1972combinatorial}. Let $X = \{x_1, \cdots, x_N\} \subset \Omega \subset \mathbb R^d$ be a finite set of points and $\mathcal F$ a class of real-valued functions on $\Omega$. The class $\mathcal F$ is defined to shatter the points set $X$ is given any signs 
\begin{equation*}
    (\varepsilon_1, \cdots, \varepsilon_N) \subset \{\pm 1\}^N,
\end{equation*}
there exists an $f\in \mathcal F$ such that 
\begin{equation*}
    {\rm sgn}\left(f(x_i)\right) = \varepsilon_i.
\end{equation*}
Here we are using the notation that
\begin{equation*}
    {\rm sgn}(x) = \begin{cases}
        -1 \quad &x < 0, \\
        1 \quad &x \ge 0.
    \end{cases}
\end{equation*}

In recent research presented in~\cite{siegel2023sharp}, the author established the subsequent theorem regarding the capacity of ReLU DNNs to shatter points.
\begin{theorem}[Theorem 1 of \cite{siegel2023sharp}]\label{thm:shattering}
Let $\delta >0$ and suppose that $\mathcal N_{N_{1:L}}$ shatters every subset $X = \{x_1,\cdots,x_N\} \subset [0,1]$ which satisfies $|x_i-x_j| \ge \delta $ for $i\neq j$. There exists a constant $c$ independent from $N$ such that if 
\begin{equation*}
\delta \le e^{-cN}    
\end{equation*}
we must have
\begin{equation*}
\sum_{i=1}^L (N_{i-1}+1)N_i + N_L \ge N/6,
\end{equation*}
where $N_0 = 1$.
\end{theorem}

Since shattering points is inherently less demanding than CPwL interpolation, we have the following result showing that at least $\mathcal O(N)$ parameters are needed to recover all CPwL functions with $N$ segments.
\begin{theorem}
If the following inclusion 
\begin{equation}\label{eq:AsubsetN}
   \mathcal A_N \subset \mathcal N_{N_{1:L}}
\end{equation}
holds on $[0,1]$ for any $N$, we must have
\begin{equation*}
    \sum_{i=1}^L (N_{i-1}+1)N_i + N_L \ge N/6,
\end{equation*}
where $N_0 = 1$.
\end{theorem}
\begin{proof}
Given that $\mathcal A_N$ contains all CPwL functions on $[0,1]$ with $N$ segments and permits any grid point distribution, let us consider an arbitrary grid $\mathcal T =\left\{0 = x_0 < x_1 \cdots < x_{N-1} < x_N =1 \right\} \subset [0,1]$ which satisfies $|x_1| \ge e^{-cN}$ as $N \to \infty$.
Here, we note that $\mathcal A_N$ can shatter $\mathcal T$ since it captures all CPwL functions on $\mathcal T$ for any grid point distribution and given any set of target function values ${y_i}{i=0}^N$. Thereby, $\mathcal N{N_{1:L}}$ should similarly be capable of shattering every possible $\mathcal T$ if the inclusion in \eqref{eq:AsubsetN} holds. As a direct corollary of Theorem~\ref{thm:shattering}, we must have $\sum_{i=1}^L (N_{i-1}+1)N_i + N_L \ge N/6$.
\end{proof}

\section{A new approximation result of ReLU DNNs based on KST}
\label{sec:approx}
In this section, we investigate approximation results of ReLU DNNs with arbitrary depth and width, drawing upon the Kolmogorov Superposition Theorem. We initiate by presenting a version of KST as detailed in \cite{lorentz1966approximation}. Leveraging the framework of KST, we establish a comprehensive approximation result for any continuous functions defined on the interval $[0,1]^d$. Subsequently, by introducing a specific function class based on the function decomposition in KST, we obtain an improved approximation rate that effectively overcomes the curse of dimensionality.

\begin{theorem}[Kolmogorov Superposition Theorem \cite{lorentz1966approximation}]\label{thm:kolmogorov} 
For any $d \in \mathbb N$, $d\ge 2$, there exist irrational numbers $0<\lambda_i\leq 1$ for $i=1,2,\cdots, d$ and strictly increasing functions $\phi_k$ (independent from $f$) with Lipschitz constant $\log_{10}2$  
defined on $[0,1]$ for $k=0, 1, \cdots, 2d$ such that 
for every continuous function $f$ defined on $[0,1]^d$, 
there exists a continuous function $g: [0,d] \mapsto \mathbb R$ such that 
\begin{equation*}
    f(x_1,\cdots,x_d)=\sum_{k=0}^{2d} g\left(\sum_{i=1}^d\lambda_i\phi_k(x_i)\right).
\end{equation*}
\end{theorem}

We re-emphasize that the K-inner functions, $\phi_k$ for $k=0:2d$ are independent of $f$. However, the K-outer function, $g$, determined by $f$, lacks uniqueness even when all K-inner functions $\phi_k$ are given~\cite{liu2015kolmogorov}. This non-uniqueness of outer functions doesn't conflict with the principles of KST. For any given $f$, its outer function can be selected in a specific manner. For instance, in the proofs of \cite{kolmogorov1957representation,lorentz1966approximation,sprecher1965structure}, a unique and distinct K-outer function $g$ is constructed for $f$ in a specific way described by the authors. 
In our study, we consistently operate under the presumption that the K-outer function $g$ for $f$ follows the formulation presented in \cite{lorentz1966approximation}.

We then have the following error estimate for approximating any continuous function on $[0,1]^d$ using ReLU DNNs with general depth and width.
\begin{theorem}\label{thm:DNNKST}
    For any continuous function $f(\bm x)$ on $[0,1]^d$, we have
    \begin{equation*}
        \inf_{\widetilde f \in \mathcal N_{(2d+1)(6N+4)}^{2(L+1)}} \|f-\widetilde f\|_{L^\infty\left([0,1]^d\right)} \le (6d+3)\omega_g\left(C_dN^{-2}L^{-1}\right)
    \end{equation*}
    where $C_d = d\log_{10}2$ and $g: [0,d] \mapsto \mathbb R$ is the K-outer function depending on $f(\bm x)$ and $\phi_k$ given by KST in \ref{thm:kolmogorov}.
\end{theorem}
\begin{proof}
For any $\bm x=(x_1,x_2,\cdots,x_d) \in [0,1]^d$, let us define 
\begin{equation*}
    \widetilde f(\bm x) = \sum_{k=0}^{2d} \widetilde{g}\left(\sum_{i=1}^d\lambda_i\widetilde{\phi}_k(x_i)\right)
\end{equation*}
where $\widetilde{\phi}_k: [0,1] \mapsto \mathbb R$ and $\widetilde g: [0,d] \mapsto \mathbb R$. In particular, we construct 
\begin{equation*}
   \widetilde{\phi}_k = \mathcal I_{\mathcal T_1}(\phi_k)  \in \mathcal N_{6N+4}^{L+1} \quad \forall k=0:2d
\end{equation*}
and
\begin{equation*}
    \widetilde{g} = \mathcal I_{\mathcal T_2}(g)  \in \mathcal N_{5(2d+1)N+4}^{L+1} \subset N_{(2d+1)(6N+4)}^{L+1},
\end{equation*}
where $\mathcal T_1 =\left\{0,  \frac{1}{N^2L}, \cdots, \frac{N^2L-1}{N^2L}, 1 \right\}$
is the uniform grid on $[0,1]$ with $N^2L$ segments and  
$\mathcal T_2 =\left\{0,  \frac{1}{(2d+1)^2N^2L}, \cdots, \frac{(2d+1)^2N^2L-1}{(2d+1)^2N^2L}, 1 \right\}$
is the uniform grid on $[0,d]$ with $(2d+1)^2N^2L$ segments. By composite these networks together, we have $\widetilde f(x) \in \mathcal N_{(2d+1)(6N+4)}^{2(L+1)}$.

Given the properties of piecewise linear interpolation in Properties~\ref{pro:interpolation}, we have the following estimates 
\begin{equation*}
\begin{aligned}
   \left\|g - \widetilde g \right\|_{L^{\infty}([0,d])} &\le 2\omega_{g}\left(d((2d+1)N)^{-2}L^{-1}\right)  \\
   &\le 2\omega_{g}\left(N^{-2}((2d+1)L)^{-1}\right)
\end{aligned}
\end{equation*}
   and
\begin{equation*}
 \left\|\phi - \widetilde \phi_k \right\|_{L^{\infty}(I)} \le \log_{10}2 N^{-2}L^{-1} \quad \forall k=0:2d.
\end{equation*}
Furthermore, we have
\begin{equation*}
 \left\| \widetilde \phi_k \right\|_{L^{\infty}([0,1])} =  \left\|  \mathcal I_{\mathcal T_1}(\phi_k) \right\|_{L^{\infty}([0,1])} \le  \left\|\phi\right\|_{L^{\infty}([0,1])}
\end{equation*}
for all $k=0:2d$. This means $\sum_{i=1}^d\lambda_i\widetilde{\phi}_k(x_i) \in [0,d]$ which is is necessary in the following estimate.

Now, for any $\bm x = (x_1,x_2,\cdots,x_d) \in [0,1]^d$, we have
\begin{align*}
    &\left|f(\bm x) - \widetilde f(\bm x)\right| \\
    = &\left| \sum_{k=0}^{2d} g\left(\sum_{i=1}^d\lambda_i\phi_k(x_i)\right) - \sum_{k=0}^{2d} \widetilde{g}\left(\sum_{i=1}^d\lambda_i\widetilde{\phi}_k(x_i)\right) \right| \\
    \le &\sum_{k=0}^{2d} \left|g\left(\sum_{i=1}^d\lambda_i\phi_k(x_i)\right) -\widetilde{g}\left(\sum_{i=1}^d\lambda_i\widetilde{\phi}_k(x_i)\right) \right| \\
    \le &\sum_{k=0}^{2d} \left|g\left(\sum_{i=1}^d\lambda_i\phi_k(x_i)\right) -g\left(\sum_{i=1}^d\lambda_i\widetilde{\phi}_k(x_i)\right) \right|  \\
    + &\sum_{k=0}^{2d} \left|g\left(\sum_{i=1}^d\lambda_i\widetilde \phi_k(x_i)\right) -\widetilde{g}\left(\sum_{i=1}^d\lambda_i\widetilde{\phi}_k(x_i)\right) \right| \\
    \le &\sum_{k=0}^{2d}\left( \omega_g\left(\sum_{i=1}^d\lambda_i\left(\phi_k(x_i) - \widetilde \phi_k(x_i)\right) \right) +  \left\|g - \widetilde g \right\|_{L^{\infty}([0,d])} \right) \\
    \le &\sum_{k=0}^{2d}\left( \omega_g\left(d\|\phi_k - \widetilde \phi_k\|_{L^{\infty}([0,1])}\right) +  \left\|g - \widetilde g \right\|_{L^{\infty}([0,d])} \right) \\
    \le &\sum_{k=0}^{2d}\left( \omega_g\left( C_d N^{-2}L^{-1}\right) + 2\omega_g\left(N^{-2}((2d+1)L)^{-1}\right) \right) \\
    \le & (6d+3)\omega_g\left( C_d N^{-2}L^{-1}\right),
\end{align*}
where $C_d = d \log_{10}2$. The last inequality holds since $C_d = d \log_{10}2 \ge (2d+1)^{-1}$ for any $d\ge1$.
\end{proof}

Given that the K-inner functions $\phi_k$ for $k=0:2d$ are independent of $f$, we define the following function class on $[0,1]^d$:
\begin{equation}\label{eq:Kc}
    K_C := \left\{ f(\bm x)=\sum_{k=0}^{2d} g\left(\sum_{i=1}^d\lambda_i\phi_k(x_i)\right) ~:~
    g \in {\rm Lip}_C\left([0,d]\right) \right\}.
\end{equation}
This is called the K-Lipschitz continuous function set with a K-Lipschitz constant of $C$ as listed in \cite{lai2021kolmogorov}. In addition, in a recent study by the authors of~\cite{lai2021kolmogorov}, an approximation rate of $\mathcal O(P^{-\frac{1}{2}})$ for functions in $K_C$ was achieved using ReLU DNNs with two hidden layers and $\mathcal O(P)$ parameters. 

Building upon Theorem~\ref{thm:DNNKST}, we have a better error estimate for functions in $K_C$, considering a more general architecture of ReLU DNNs:
\begin{corollary}\label{cor:KcApprox}    
For any function $f \in K_C$, we have
\begin{equation*}
    \inf_{\widetilde f \in \mathcal N_{(2d+1)(6N+4)}^{2(L+1)}} \|f-\widetilde f\|_{L^\infty([0,1]^d)} = (6d+3)CC_dN^{-2}L^{-1},
\end{equation*}
    where  $C_d = d\log_{10}2$ and $g: [0,d] \mapsto \mathbb R$ is the K-outer function given by KST in \ref{thm:kolmogorov}.
\end{corollary}

It's worth noting that the total parameter count in $\mathcal N_{(2d+1)(6N+4)}^{2(L+1)}$ is approximately $2(L+1) \times \left((2d+1)(6N+4)^2\right) = \mathcal O(N^2L)$. Hence, this outcome indicates that functions within $K_C$ can be approximated at an error rate of $\mathcal O\left(P^{-1}\right)$ when employing ReLU DNNs having $\mathcal O(P)$ overall parameters across general depths and widths.

\section{Conclusion}
\label{sec:conclusion}
In conclusion, our study has shed further light on the optimal expressive power of ReLU DNNs and its utility in approximation, specifically employing the Kolmogorov Superposition Theorem. We provided a constructive proof, asserting that ReLU DNNs with $L$ hidden layers and $N$ neurons per layer can represent any CPwL functions on the interval $[0, 1]$ comprising $\mathcal O(N^2L)$ segments. Additionally, we demonstrated the optimality of this construction in relation to the number of parameters in the DNNs, a feat accomplished through an exhaustive analysis of the shattering capacity of ReLU DNNs. Finally, we harnessed the Kolmogorov Superposition Theorem to improve the approximation rate of ReLU DNNs of variable widths and depths when approximating continuous functions in high-dimensional spaces. 

This work gives rise to a variety of interesting and meaningful questions related to the expressivity and approximation of ReLU DNNs. For instance, our exact representation result in Theorem~\ref{thm:interpolation} could simplify the proofs presented in~\cite{shen2019nonlinear,shen2020deep,lu2021deep,shen2022optimal}. Additionally, we aim to extend the construction in Theorem~\ref{thm:interpolation} for ReLU DNNs with arbitrary depth and width from one-dimensional space to high-dimensional cases. If proven, this result could be further combined with the bit-extraction technique, leading to new approximation results.
Furthermore, our Theorem~\ref{thm:DNNKST} and Corollary~\ref{cor:KcApprox} lay the foundation for stimulating inquiries into pure mathematical aspects of KST. A pertinent question would be how the modulus of continuity of $g$ in the KST is influenced by the original function $f$ and the associated construction. Another point of interest is the functional class of K-Lipschitz continuous functions as defined in~\eqref{eq:Kc}. How extensive is this class, and can it serve as a sufficient hypothesis space in the domains of machine learning or scientific computing? These inquiries and their potential answers will likely advance our understanding of these fields.

\section*{Acknowledgements}
We would like to thank Jinchao Xu, Zuowei Shen, and Jonathon Siege for helpful discussions. This work was supported by the KAUST Baseline Research Fund.

\bibliographystyle{abbrv} 
\bibliography{main.bib}

\newpage
\appendix
\section{Function diagrams from Proof of Theorem~\ref{thm:interpolation}}
\label{sec:appendix}
On this page, we show diagrams for $\phi_0(x)$, $\phi_1(x)$, and $\phi_{N-1}(x)$ as defined in \eqref{eq:phi0}, \eqref{eq:phi1}, and \eqref{eq:phiN-1}, respectively.
\begin{figure}[H]
	\centering
	\includegraphics[width=.75\textwidth]{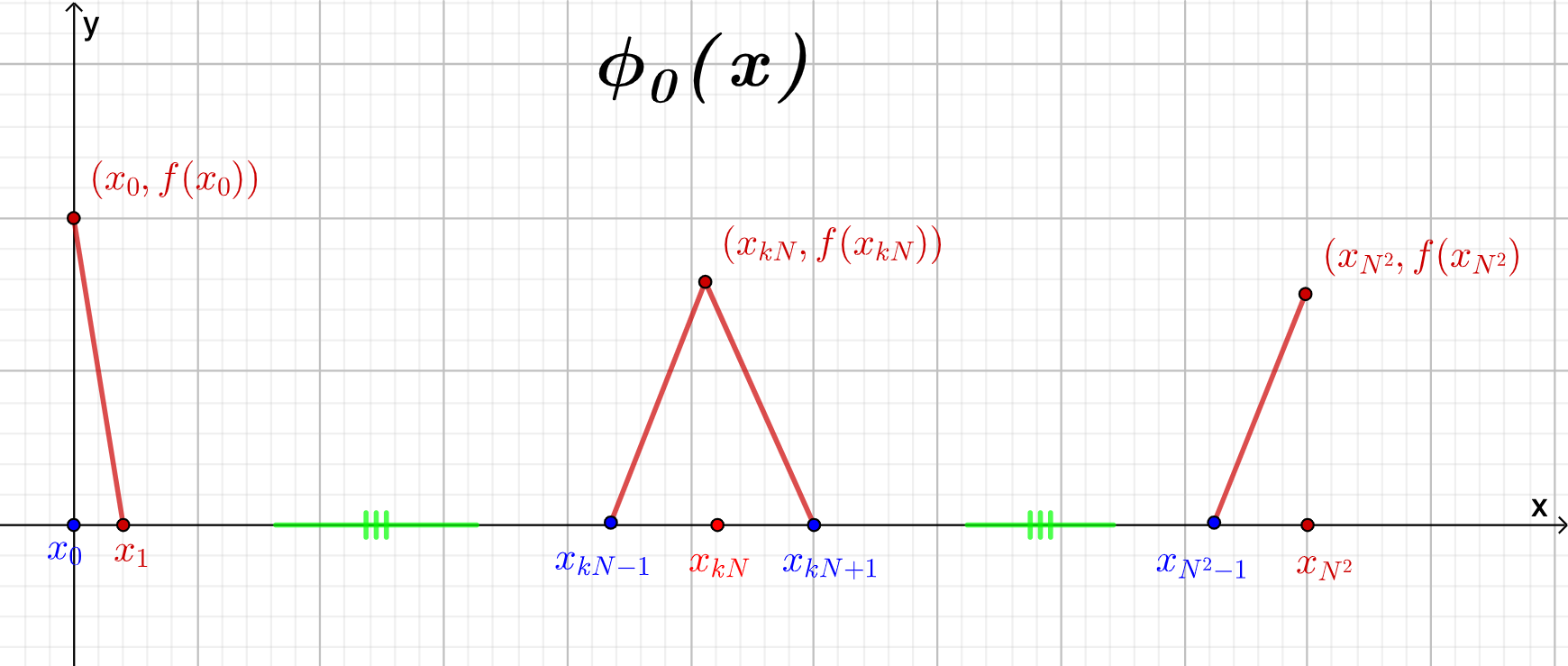} \\ \vspace{3pt}
        \includegraphics[width=.75\textwidth]{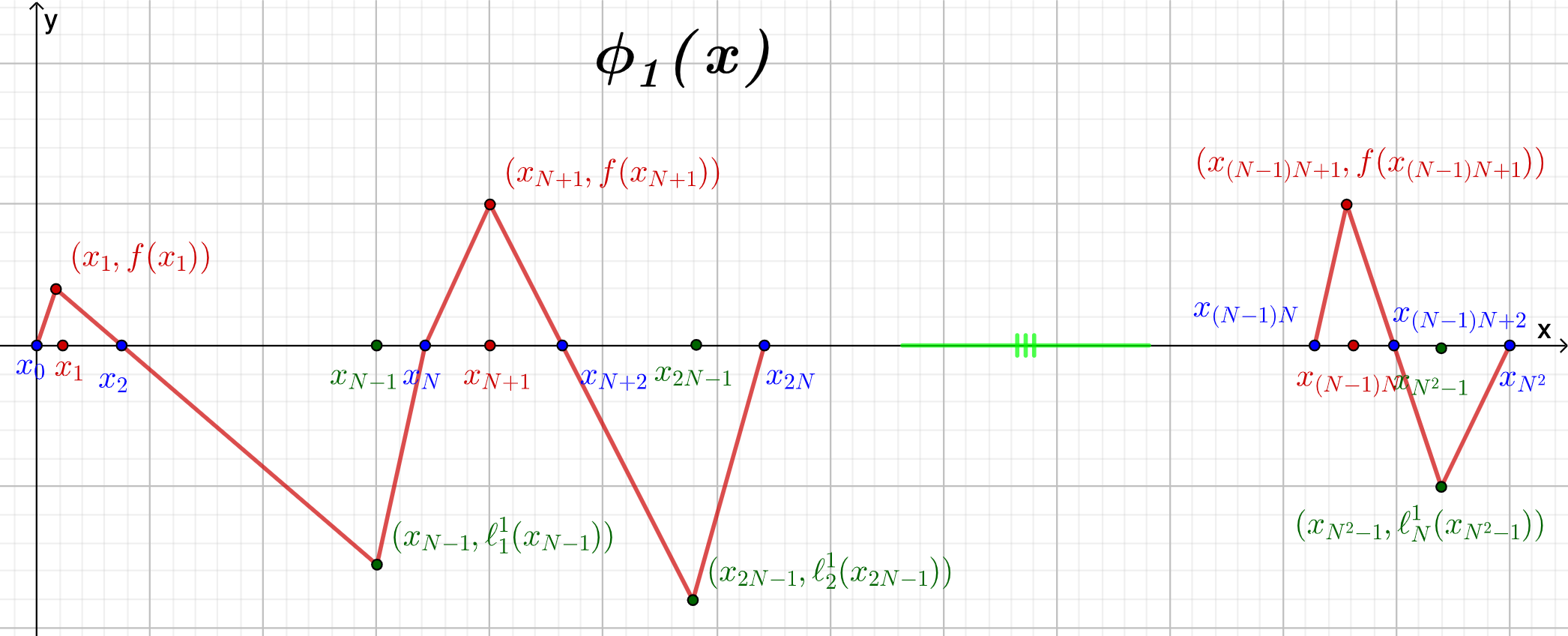} \\ \vspace{3pt}
        \includegraphics[width=.75\textwidth]{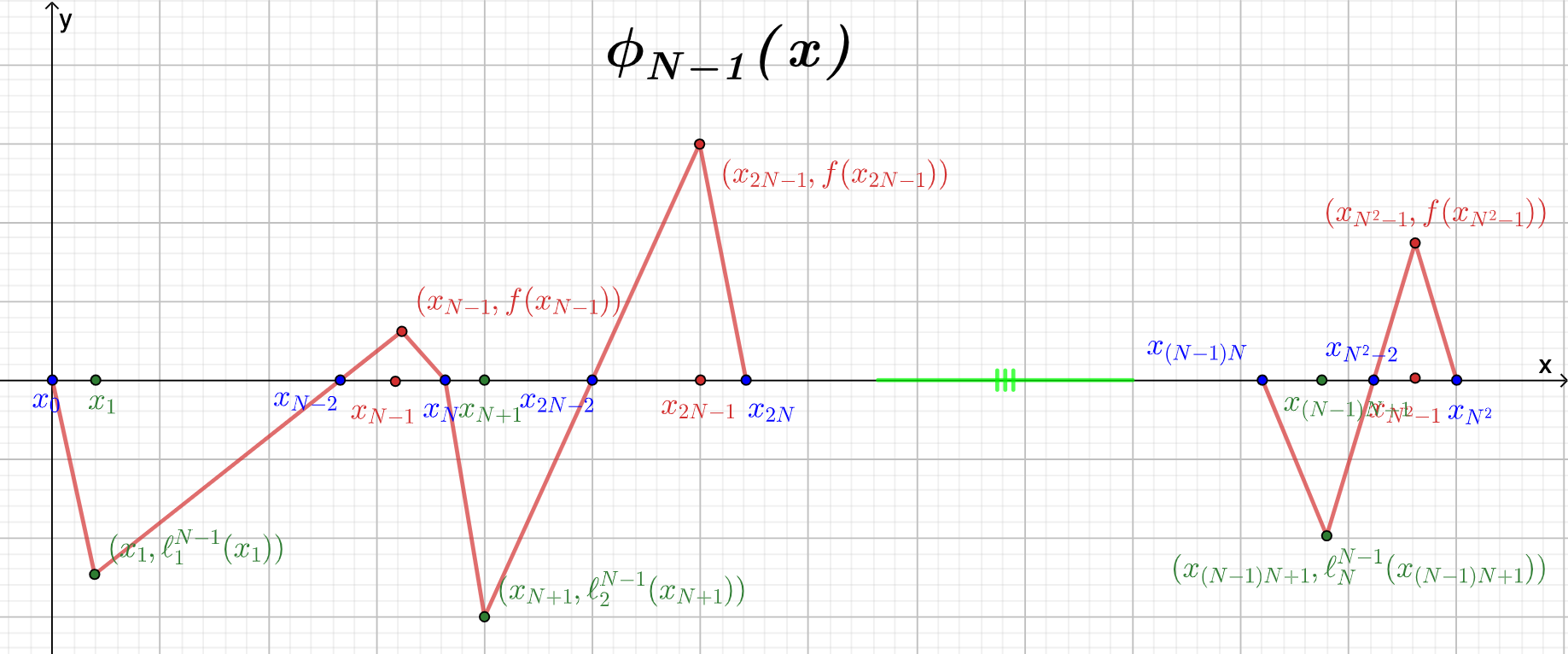} 
	\caption{Diagrams of $\phi_0(x)$, $\phi_1(x)$, and $\phi_{N-1}(x)$.}
	\label{fig:phi01N-1}
\end{figure}

\newpage
On this page, we present diagrams for $\phi_k^-(x)$, $\phi_k^+(x)$, and $\phi_k(x)$ as defined in \eqref{eq:phik-_ineq}, \eqref{eq:phik+}, and \eqref{eq:phik}, respectively.
\begin{figure}[H]
	\centering
	\includegraphics[width=.75\textwidth]{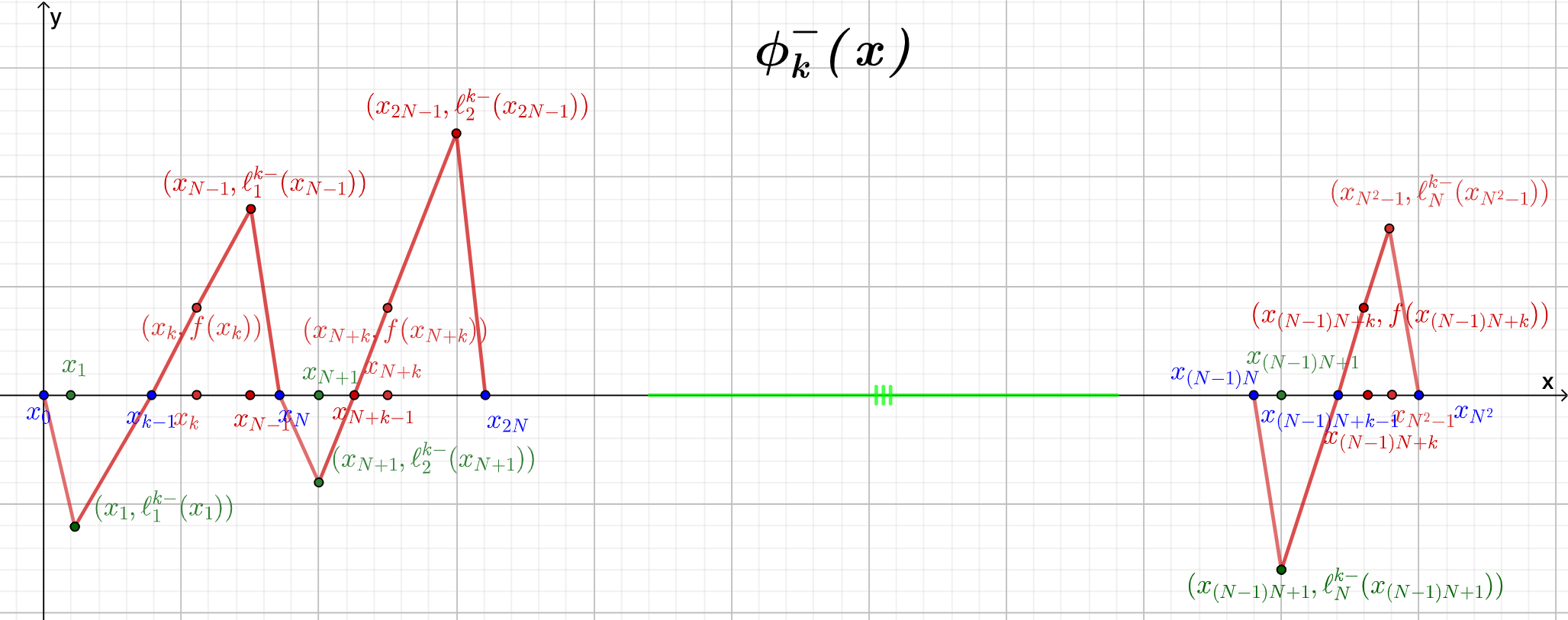} \\ \vspace{5pt}
        \includegraphics[width=.75\textwidth]{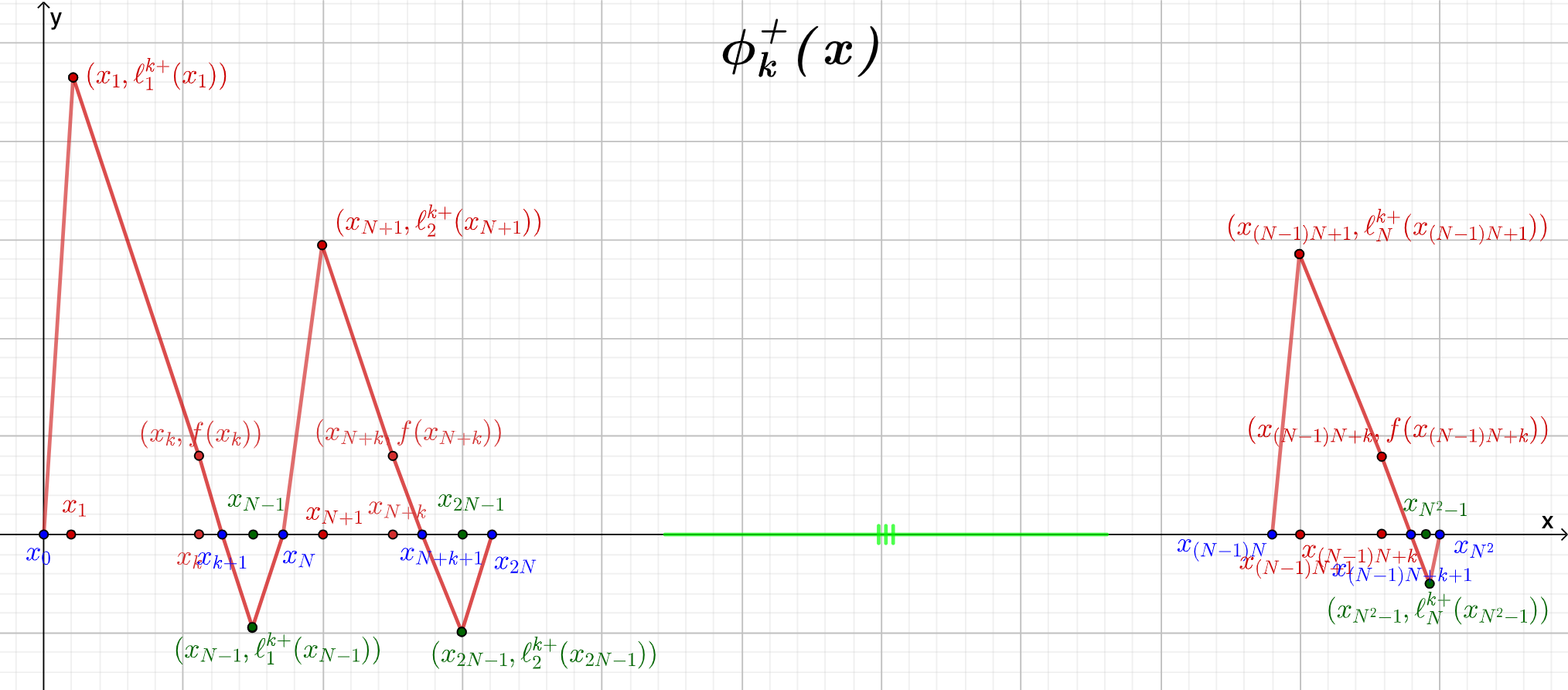} \\ \vspace{5pt} 
	\includegraphics[width=.75\textwidth]{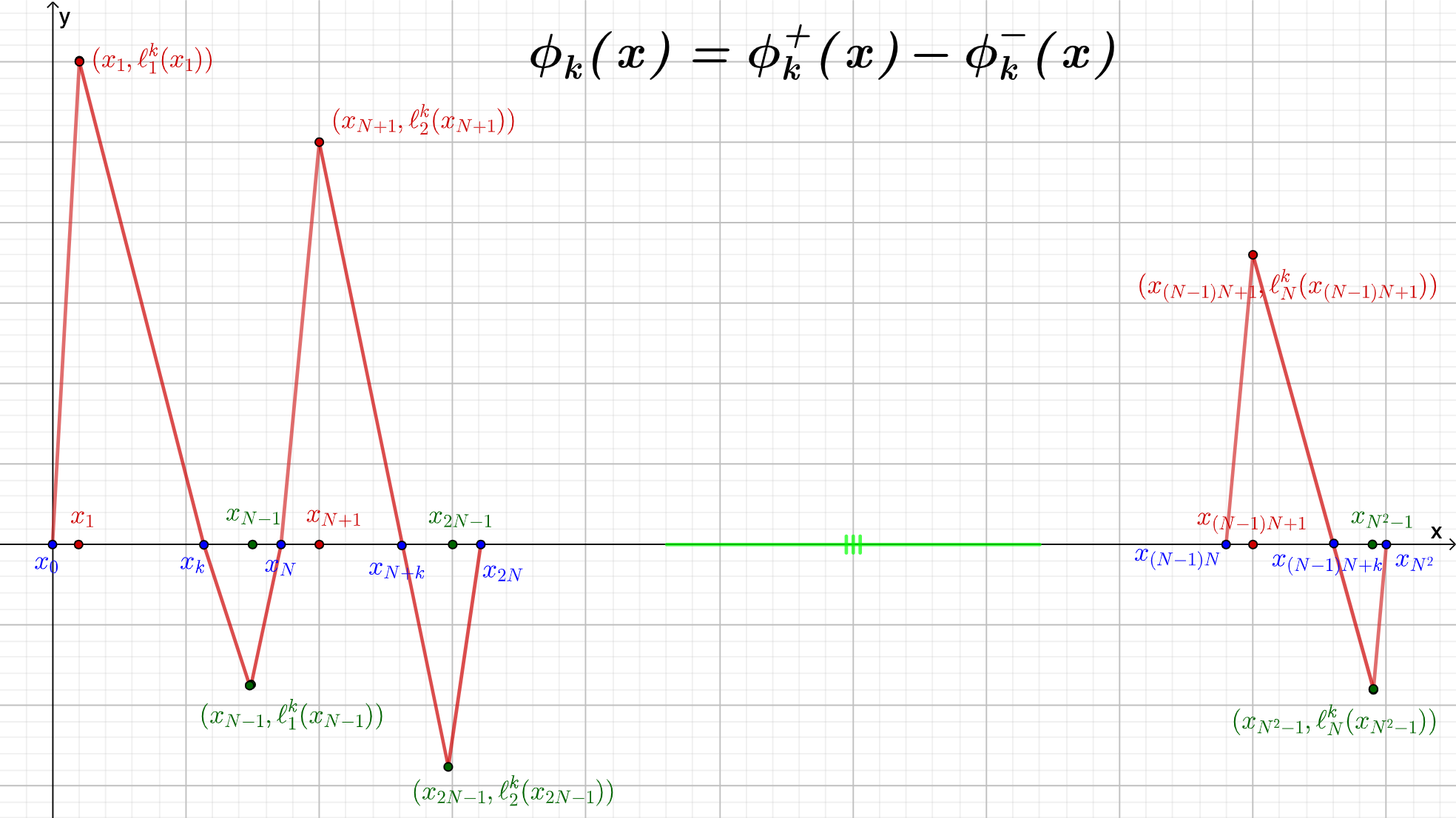}  
	\caption{Diagrams of $\phi_k^-(x)$, $\phi_k^+(x)$, and $\phi_k(x)$.}
	\label{fig:phi_k-+}
\end{figure}

\newpage
On this page, we show the diagram for $\varphi_k(x)$ as defined in \eqref{eq:varphik_interp}.
\begin{figure}[H]
	\centering
 \includegraphics[width=.95\textwidth]{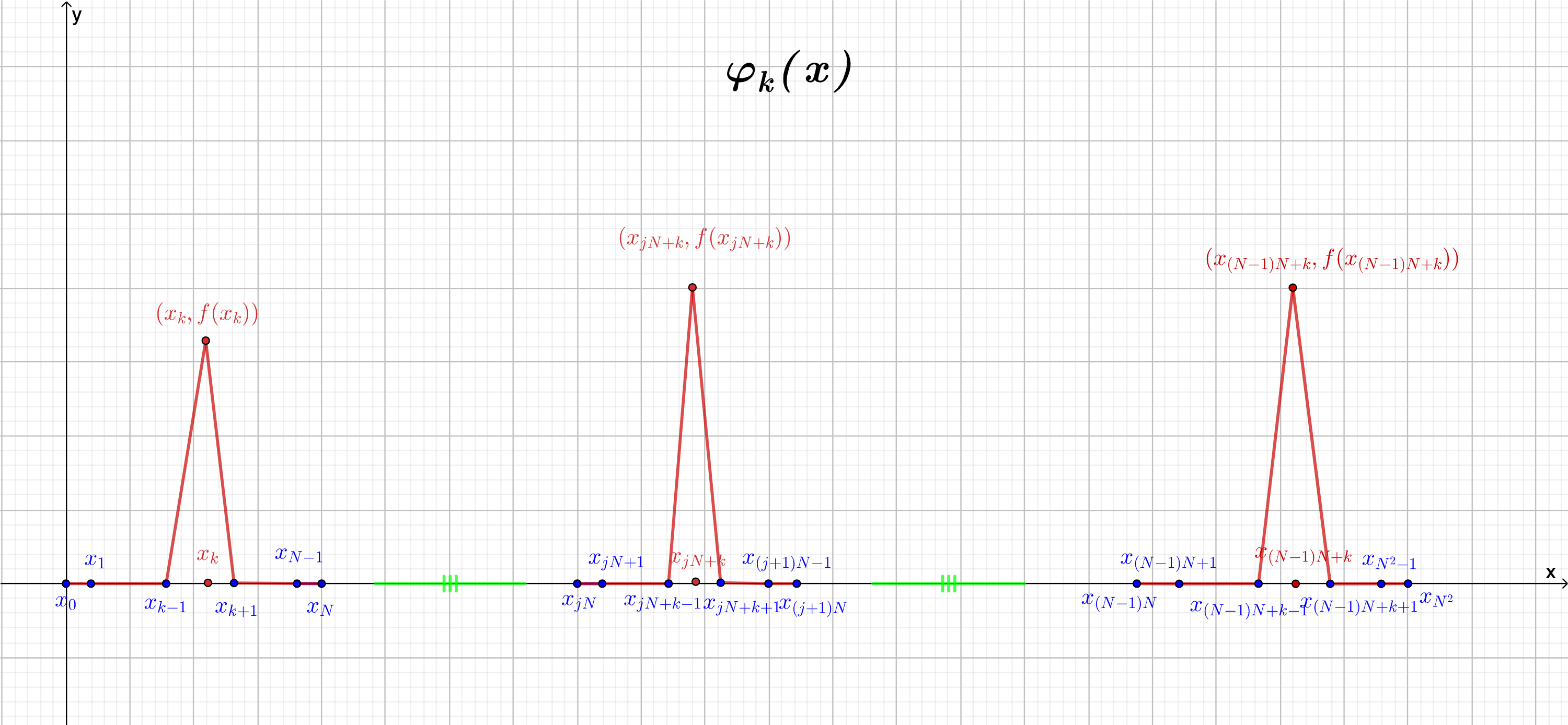} 
	\caption{Diagram of $\varphi_k(x)$.}
	\label{fig:varphi_k}
\end{figure}

\end{document}